\documentclass[11pt]{article}

\usepackage[preprint]{acl}

\usepackage{times}
\usepackage{latexsym}

\usepackage{amsmath,amsthm,amssymb,amsfonts}
\usepackage{algorithm,algorithmic}
\usepackage{subcaption}
\usepackage{wrapfig} % add this in your preamble
\usepackage{booktabs}
\usepackage{adjustbox}

\newtheorem{theorem}{Theorem}%[section]

\usepackage{multirow}
\usepackage{tcolorbox}
\definecolor{lightgray}{gray}{0.9}
\definecolor{headergray}{gray}{0.8}
\definecolor{darkgray}{gray}{0.6}
\usepackage{colortbl}
\usepackage{pifont}
\newcommand{\cmark}{\ding{51}}%
\newcommand{\xmark}{\ding{55}}
\usepackage[T1]{fontenc}
\usepackage[utf8]{inputenc}

\usepackage{microtype}

\usepackage{inconsolata}

\usepackage{graphicx}

\title{Reservoir of Importance: Learning Semi-Structured Sparsity with Differentiable Subset Sampling}

\author{
 \textbf{Ha Dinh\textsuperscript{1}} \thanks{Equal contribution.} \enspace
 \textbf{Xuan Duy Ta\textsuperscript{1}} \footnotemark[1] \enspace
 \textbf{Khoat Than\textsuperscript{2}} \enspace
 \textbf{Khac-Hoai Nam Bui\textsuperscript{1}} \thanks{Corresponding author.}
\\
 \textsuperscript{1}Viettel AI, Viettel Group, Vietnam \\
 \textsuperscript{2}Hanoi University of Science and Technology, Hanoi, Vietnam
\\
 {\small
 \texttt{\{hadv17,duytx\}@viettel.com.vn}, \texttt{khoattq@soict.hust.edu.vn}, \texttt{nambkh@viettel.com.vn}
 }
}

\begin{document}
\maketitle
\begin{abstract}
Semi-structured $N$:$M$ sparsity has emerged as a practical direction for accelerating large language models (LLMs). However, existing learnable-mask approaches incur substantial parameter and memory overhead, limiting their scalability to large models and aggressive sparsity regimes. In this work, we revisit semi-structured pruning from a perspective that reconciles efficiency with scalability. We propose Reservoir of Importance (RoI)~\footnote{The source code is publicly available at: \url{https://github.com/havisdino/roi}}, a lightweight semi-structured pruning framework that learns sparsity masks through differentiable subset sampling. Unlike prior methods that model full categorical distributions over all feasible $N$:$M$ patterns, RoI introduces a compact-logit parameterization for sparsity mask learning and performs sampling without replacement to select masks, thereby reducing trainable parameters from combinatorial complexity to $\mathcal{O}({M})$. As a result, RoI requires 1.5--8.75$\times$ fewer learnable parameters and significantly lower memory cost, while remaining fully aligned with hardware-friendly sparsity patterns. Extensive evaluations across multiple scales of the Qwen2.5 LLM family (0.5-7B parameters) demonstrate that RoI achieves competitive performance with strong memory efficiency, stability, and scalability to more aggressive $N$:$M$ sparsity patterns, offering a practical path toward efficient LLM deployment.
\end{abstract}

\section{Introduction}
\begin{figure}[!t]
  \centering
  \includegraphics[width=\columnwidth]{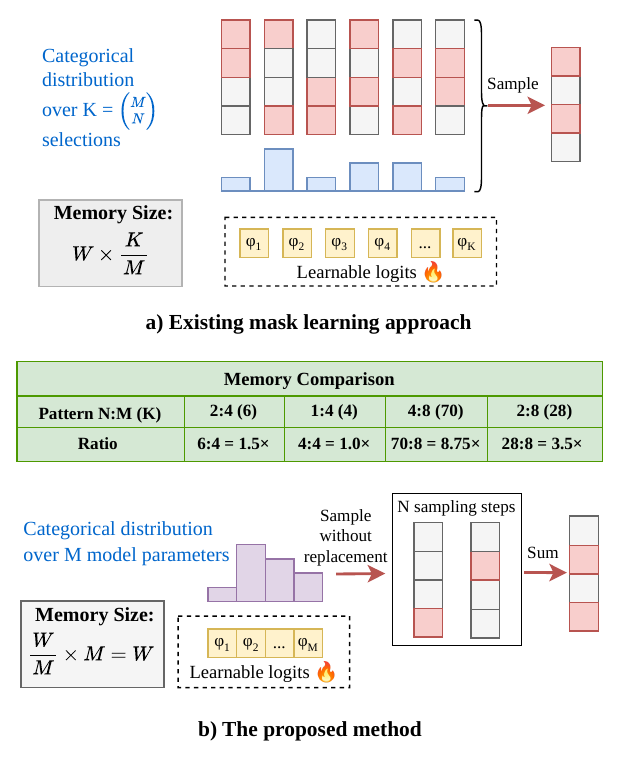}
  \caption{Learnable semi-structured $N$:$M$ sparsity methods: a) modeling the mask selection process using a categorical distribution over feasible masks, and b) our proposed method by learning to sample subsets without replacement of model parameters. The proposed method is more memory efficient than previous works for most practical $N$:$M$ sparsity patterns. The memory advantage becomes more pronounced as $M$ increases or when $N$ is around $M/2$.}
  \label{fig1}
\end{figure}
Pruning techniques play a central role in developing sparse LLMs that reduce memory footprint and improve inference efficiency. Post-training pruning methods generally fall into three categories:
(i) unstructured pruning, which removes individual weight parameters without considering architectural constraints~\cite{Sun0BK24};
(ii) structured pruning, which removes entire components such as neurons, attention heads, or layers~\cite{XiaGZ024, LeDWW0Y025}; and
(iii) semi-structured pruning, which balances flexibility and regularity by enforcing hardware-friendly sparsity patterns such as $N$:$M$ sparsity~\cite{FangYMHPK0W24, HuangHJZC25}.

In this work, we focus on semi-structured pruning because it removes redundant weights while maintaining regular sparsity patterns that are compatible with modern accelerators. In particular, enforcing $N$:$M$ sparsity patterns~\cite{HubaraCIBNS21} provides principled trade-offs between sparsity flexibility and hardware efficiency, enabling practical inference acceleration.

Recent progress in semi-structured pruning has increasingly focused on learnable-mask approaches that directly optimize $N$:$M$ sparsity patterns through gradient-based training. Representative methods (e.g., MaskLLM~\cite{FangYMHPK0W24} and HyperPrune~\cite{sun2026learning}) formulate each pruning group as a categorical distribution over all feasible $N$:$M$ configurations, which demonstrate strong pruning performance and generalization. However, the major challenge of this approach is the combinatorial parameterization required to model a multinomial distribution of size $\binom{M}{N}$ for each weight group. Specifically, as $M$ increases, this design imposes substantial memory and parameter overhead, as demonstrated in Figure~\ref{fig1}~(a).
% Moreover, adaptive sparse-training variants require repeated regrowth cycles or dense mask structures, further amplifying training cost and complexity.
Motivated by this limitation, we propose a lightweight alternative termed Reservoir of Importance (RoI), which replaces full categorical modeling with differentiable subset sampling. Specifically, instead of parameterizing the probability of every possible $N$:$M$ pattern, RoI samples $N$ entries without replacement from a compact logit vector of size $M$ using Weighted Reservoir Sampling and Gumbel-Top-$K$ relaxation. This reduces the learnable mask parameters from  $\mathcal{O}(\binom{M}{N})$ to a simple $\mathcal{O}({M})$ per group while retaining differentiability and hardware alignment, as shown in Figure~\ref{fig1}~(b). By avoiding combinatorial distributions and relying on efficient subset-sampling dynamics, the proposed method enables scalable, stable, and memory-efficient mask learning for semi-structured pruning of modern LLMs. 

\section{Related Work}
% Pruning LLMs is a critical optimization technique that removes less significant or redundant parameters, such as weights or neurons, from the neural network architecture. This process reduces model size, computational complexity, and memory requirements, thereby improving inference speed and enabling deployment on resource-constrained devices. Pruning methods for LLMs are broadly categorized into structured, unstructured, and semi-structured approaches, each with distinct characteristics and trade-offs~\cite{ChengZS24a}. 

% \textbf{Structured pruning} involves the elimination of entire architectural components, such as layers or attention heads, to improve computational efficiency~\cite{AshkboosCNHH24, XiaGZ024, AnZYTW24, LiuYCZWDD25, LeDWW0Y025}. This approach simplifies the model structure, making it more amenable to hardware optimization. However, it frequently results in substantial performance degradation, necessitating extensive retraining to restore model functionality. 

% \textbf{Unstructured pruning} targets individual weights based on their significance, enabling high performance even at elevated sparsity levels~\cite{DongLTLP0024, Sun0BK24}. Despite its efficacy in preserving model accuracy, the irregular sparsity patterns produced are often incompatible with hardware acceleration, limiting its practical applicability in deployment scenarios. 
% \subsection{Semi-structured pruning}

Semi-structured pruning has emerged as an effective compromise between structured pruning~\cite{AnZYTW24, LiuYCZWDD25} and unstructured pruning~\cite{DongLTLP0024, Sun0BK24}. By enforcing regular sparsity patterns (i.e., $N$:$M$ sparsity), semi-structured pruning enables efficient hardware acceleration while preserving model accuracy~\cite{HubaraCIBNS21}.

\subsection{Saliency-Based Methods}

Saliency-based approaches, including SparseGPT \cite{FrantarA23} and Wanda~\cite{Sun0BK24}, rely on a small calibration dataset, typically a subset of the pre-training data, to approximate the knowledge encoded in a pre-trained language model. These methods estimate the importance of individual weights or weight groups using criteria such as weight magnitude, gradient information, or second-order statistics (e.g., Hessian approximations inspired by Optimal Brain Damage~\cite{CunDS89} and Optimal Brain Surgeon~\cite{HassibiS92}), and prune parameters accordingly. Although computationally efficient, the effectiveness of these importance estimates strongly depends on the underlying saliency criterion and the representativeness of the calibration data. In practice, the limited calibration set may fail to adequately capture the rich and diverse knowledge embedded in large language models, potentially leading to suboptimal pruning decisions.

\subsection{Learning-Based Methods}

Learning-based semi-structured pruning has recently gained attention due to its ability to optimize pruning decisions directly through gradient-based training, rather than relying on handcrafted importance metrics. Under the $N$:$M$ sparsity constraint, the objective is to retain exactly $N$ non-zero weights within each group of $M$ parameters, thereby producing hardware-friendly sparsity patterns while minimizing performance degradation~\cite{ZhouMZLZYSL21}. A prominent line of work formulates mask learning as an optimization problem over categorical distributions that enumerate all feasible $N$:$M$ configurations. For example, MaskLLM~\cite{FangYMHPK0W24} models each group's pruning mask as a multinomial distribution and employs Gumbel-Softmax sampling to enable differentiable training. This approach achieves strong pruning performance and good generalization across tasks; however, it incurs substantial computational and memory overhead due to the $\mathcal{O}(\binom{M}{N})$ parameterization per group. On the other hand, ProxSparse~\cite{LiuSJPHS0K25} reduces training complexity via regularized optimization and smaller learning samples. Despite these improvements, challenges related to large-scale learning efficiency and performance generalization remain open.

In contrast to prior work, our study follows a large-scale learning-based paradigm while substantially reducing the parameterization cost of mask learning. Specifically, we improve the categorical formulation with an $\mathcal{O}({M})$ parameterization per group by leveraging weighted reservoir sampling. This design achieves significant memory savings while retaining high pruning accuracy, making it well-suited for large-scale semi-structured pruning, as detailed in the following section.

\section{Preliminaries}
\subsection{Weighted Reservoir Sampling}
Weighted Reservoir Sampling (WRS)~\cite{EfraimidisS06} is an extension of the Reservoir Sampling class of algorithms~\cite{Vitter85}, which aims to sample $K$ elements from a population of size $L$. In WRS, each item is associated with a positive weight, and items with larger relative weights are preferentially sampled. Formally, consider a finite population $\mathcal{X} = \{x_1, x_2, \ldots, x_L\}$ equipped with weights $\boldsymbol{w} = [w_1, w_2, \ldots, w_L]$, WRS produces an ordered subset $\mathcal{Y} = \{y_1, y_2, \ldots, y_K\}$ given by:
\begin{equation}
\label{eq:orig-wrs}
\begin{aligned}
      \Pr_\text{WRS} (\mathcal{Y}|\boldsymbol{w}) = \frac{w_{y_1}}{W} \times \frac{w_{y_2}}{W - w_{y_1}} \times \ldots 
      \\
      \times \frac{w_{y_K}}{W - \sum_{j=1}^{K-1} w_{y_j}},
\end{aligned}
\end{equation}
where $W = \sum_{i=1}^L w_i$ denotes the total weight and $w_{y_i}$ is the weight of element $y_i$. This distribution corresponds to a without-replacement sampling process, where the probability of selecting a subset is proportional to its item weights.

% \subsection{Gumbel-Max Trick}
\subsection{Gumbel-Top-$K$ Trick}
Gumbel-Max~\cite{gumbel1954statistical} can be viewed as a monotonic transformation of the WRS technique, providing a reparameterization trick to sample from a categorical distribution by perturbing the log-probabilities with Gumbel noise. Consider a categorical distribution over $\{x_1, \ldots, x_L\}$, parameterized by logits $\boldsymbol{\phi} = [\phi_1, \ldots, \phi_L]$, with probabilities $\pi_i = \exp(\phi_i)/{\sum_{j=1}^L \exp(\phi_j)}$. 
The Gumbel-Max trick samples from this distribution by first associating each item with a random key obtained via Gumbel perturbation:
\begin{equation}
    \kappa_i = \phi_i + g_i, \quad g_i \overset{\text{i.i.d}}{\sim} \operatorname{Gumbel}(0,1),
\end{equation}
where each $g_i$ is independently drawn from the $\operatorname{Gumbel}(0,1)$ distribution. The sampler outputs the item $x_j$ having the largest key $\kappa_j$. The index $j$ is the output of taking $\operatorname{argmax}$ over every perturbed key values ($j = {\operatorname{argmax}_{i}}\ \kappa_i$).

Gumbel-Top-$K$~\cite{XieE19} is a generalization of the Gumbel-Max trick, selecting the top-$K$ items with the highest keys, rather than just the maximum. This corresponds to drawing $K$ items without replacement from a categorical distribution over $L$ candidates. By replacing the $\operatorname{argtop}_K$ operator with a sequence of $\operatorname{softmax}$ relaxations~\cite{Plotz018}, the sampling process becomes differentiable, thereby enabling end-to-end learning via backpropagation.

To sample a size-$K$ subset using the Gumbel-Top-$K$ trick, logits are first independently perturbed with Gumbel noise to form random keys $\kappa_i$, analogously to the Gumbel-Max trick. Subsequently, a sequence of $\operatorname{softmax}$ operations is applied to generate differentiable approximations of the selected one-hot indicators. Let $\boldsymbol{\alpha}^{(k)} = [\alpha_1, \ldots, \alpha_L]$ denote adjusted keys at sampling step $k$, which are defined recursively as follows:
\begin{equation}
    \begin{aligned}
    \boldsymbol{\alpha}^{(1)} &= [\kappa_1, \ldots, \kappa_L], \\
    \boldsymbol{\alpha}^{(k)} &= \boldsymbol{\alpha}^{(k-1)} + \log(1 - \boldsymbol{\mu}^{(k-1)}),
\end{aligned}
\end{equation}
where $\boldsymbol{\mu}^{(k-1)} = [\mu^{(k-1)}_1, \ldots, \mu^{(k-1)}_L]$ is the relaxed one-hot indicator of the item selected at step $k-1$. This representation is achieved by applying a $\operatorname{softmax}$ with temperature $\tau$ to the adjusted keys:
\begin{equation}
    \mu^{(k-1)}_i = \frac{\exp(\alpha^{(k-1)}_i / \tau)}{\sum_{j=1}^L \exp(\alpha^{(k-1)}_j / \tau)}.
\end{equation}
After $K$ iterations, the procedure produces an ordered collection of relaxed one-hot vectors $\mathcal{S} = \{ \boldsymbol{\mu}^{(1)}, \ldots, \boldsymbol{\mu}^{(K)} \}$. Summing these vectors yields a soft $K$-hot representation, and the mapping from logits $\phi_i$ to this representation is differentiable, enabling gradient-based training.

\section{Methodology}
\begin{figure*}[t]
  \includegraphics[width=\textwidth]{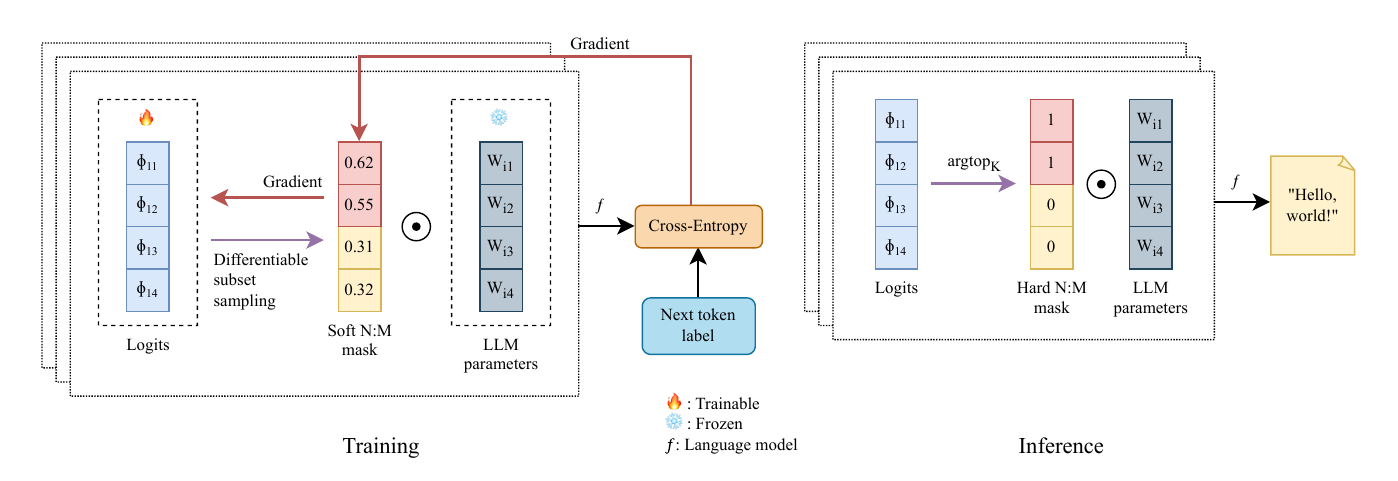}
  \caption{Overview of the RoI framework for semi-structure sparsity learning via differentiable subset sampling, illustrating the training and inference phases.}
  \label{fig:main}
\end{figure*}
\subsection{Problem Statement}
The task of determining the optimal $N$:$M$ sparsity pattern can be formulated as selecting, for each contiguous group of $M$ parameters, a length-$M$ binary mask containing exactly $N$ non-zero entries that minimizes the calibration loss. Let $G$ be the number of such groups, $\mathbf{W} = \{ \mathbf{w}_1, \ldots, \mathbf{w}_G \}$ the corresponding parameter groups, and $\mathbf{M} = \{ \mathbf{m}_1, \ldots, \mathbf{m}_G \}$ their associated masks. The resulting optimization problem is:
\begin{equation}
\label{eq:definition}
    \mathbf{M}^* = \underset{\mathbf{M}}{\operatorname{argmin}}\ \mathcal{L}_\text{CE}(\mathcal{D}; \mathbf{W} \odot \mathbf{M}).
\end{equation}
Here, $\mathcal{L}_\text{CE}$ denotes the cross-entropy loss for language modeling, $\mathcal{D}$ is the calibration dataset, and $\odot$ represents the element-wise multiplication between each weight group and its corresponding mask. This optimization problem is NP-hard due to the combinatorial search space, containing $\binom{M}{N}^G$ feasible configurations. For LLMs, the number of weight groups $G$ is extremely large, rendering exhaustive search intractable. Therefore, in the following section, we reformulate the objective as a stochastic variational optimization problem to obtain a tractable and efficient approximation.

\subsection{Proposed Method}
An overview of the proposed RoI method is presented in Figure~\ref{fig:main}.
Stochastic variational optimization~\cite{Bird018Stochastic} relies on the observation that for any arbitrary distribution $q(x)$, the expected value of a function $f(x)$ provides an upper bound for its minimum: 
\begin{equation}
    \min_x f(x) \leq \mathbb{E}_{q(x)} [f(x)].
\end{equation}
By modeling pruning masks as random variables, the optimization problem in Equation~\ref{eq:definition} can be reframed as minimizing a variational upper bound of the objective with respect to the variational distribution parameters. Specifically, we solve:
\begin{equation}
    \boldsymbol{\Phi}^* = \underset{\boldsymbol{\Phi}}{\operatorname{argmin}}\ \mathbb{E}_{\Pr(\mathbf{M}|\boldsymbol{\Phi})} [ \mathcal{L}_\text{CE}(\mathcal{D}; \mathbf{W} \odot \mathbf{M}) ],
    \label{eq:optimize}
\end{equation}
where $\boldsymbol{\Phi} = \{ \boldsymbol{\phi}_1, \ldots, \boldsymbol{\phi}_G \}$ parameterizes independent variational factors $\Pr(\mathbf{m}_1 | \boldsymbol{\phi}_1)$, \ldots, $\Pr(\mathbf{m}_G | \boldsymbol{\phi}_G)$, with joint distribution $\Pr(\mathbf{M}|\boldsymbol{\Phi}) = \prod_{i=1}^G \Pr(\mathbf{m}_i | \boldsymbol{\phi}_i)$. Under this stochastic variational formulation problem, mask sampling can be reparameterized and continuously relaxed into a differentiable mapping with respect to $\boldsymbol{\Phi}$, making it possible to learn through gradient-based methods.

\subsubsection{Variational Distribution Selection}
Since pruning masks are $N$-hot vectors of length $M$, each mask can take one of $\binom{M}{N}$ possible values. Modeling an unconstrained distribution over all possible values therefore requires $\binom{M}{N} - 1$ free parameters, which grows combinatorially with $M$. To efficiently learn masks with a manageable number of parameters, we instead model mask distributions $\Pr(\mathbf{m}_i | \boldsymbol{\phi}_i)$ using the WRS distribution (Equation~\ref{eq:orig-wrs}) over ordered subsets.
Let $\mathcal{S}_i = \{ \boldsymbol{\mu}_i^{(1)}, \ldots, \boldsymbol{\mu}_i^{(N)} \}$ be a set of $N$ one-hot vectors corresponding to the selected weights within the $i$-th group, sampled from the WRS distribution using the Gumbel-Top-$K$ trick. The probability of a mask $\mathbf{m}_i$ is then:
\begin{equation}
\label{eq:full_dist_definition}
    \Pr(\mathbf{m}_i | \boldsymbol{\phi}_i) = \sum_{\mathcal{S}_{\mathbf{m}_i}} \Pr_\text{WRS}(\mathcal{S}_{\mathbf{m}_i} | \exp(\boldsymbol{\phi}_i)),
\end{equation}
where $\mathcal{S}_{\mathbf{m}_i}$ denotes the collection of ordered subsets whose elements sum to $\mathbf{m}_i$, i.e., $\mathbf{m}_i = \sum_{j=1}^{N} \boldsymbol{\mu}_i^{(j)}$.
Exactly computing this probability is expensive and unnecessary, as the construction of $\mathbf{m}_i$ discards the ordering of items in the sampled subset. Instead, the expected loss can be computed via Monte Carlo sampling. Under this formulation, the original problem reduces to learning importance scores $\exp(\phi_{ij})$ to select salient weights so as to minimize the objective. 
Denoting the WRS distribution of an arbitrary subset $\mathcal{S}_i$ by $\Pr_\text{WRS}(\mathcal{S}_i | \boldsymbol{\phi}_i)$ for brevity, the optimization problem in Equation~\ref{eq:optimize} can be reformulated as follows:
\begin{equation}
    \boldsymbol{\Phi}^* = \underset{\boldsymbol{\Phi}}{\operatorname{argmin}}\ \mathbb{E}_{\Pr_\text{WRS} (\mathcal{S} | \boldsymbol{\Phi})} [ \mathcal{L}_\text{CE}(\mathcal{D}; \mathbf{W} \odot \mathbf{M}) ],
\end{equation}
where $\mathcal{S} = \{ \mathcal{S}_1, \ldots, \mathcal{S}_G\}$ is a collection of $G$ subsets sampled from the joint distribution $\Pr_\text{WRS}(\mathcal{S} | \boldsymbol{\Phi}) = \prod_{i=1}^G \Pr_\text{WRS} (\mathcal{S}_i | \boldsymbol{\phi}_i)$. Each $N$-hot pruning mask $\mathbf{m}_i$ in the mask collection $\mathbf{M}$ is constructed as $\mathbf{m}_i = \sum_{\boldsymbol{\mu}_i^{(j)} \in \mathcal{S}_i} \boldsymbol{\mu}_i^{(j)}$. Parameterizing masks as sums of subsets sampled from WRS-restricted distributions yields the same expected loss as sampling masks from the target distributions, as formalized in Theorem~\ref{thm:restricted-sampling}. Our approach reduces parameter complexity by viewing $N$-hot sampling as a sequential sampling without replacement process. Rather than maintaining a full categorical distribution over all $\binom{M}{N}$ possible configurations, we instead model only a single categorical distribution over the $M$ model parameters, requiring exactly $M$ parameters independent of $N$. The proposed method achieves a reduction in parameter complexity from $\mathcal{O}(\binom{M}{N})$ to $\mathcal{O}({M})$, yielding an exponential improvement in memory efficiency.

\subsubsection{Mask Selection Relaxation}
To enable differentiation of the objective with respect to the variational distributions' parameters, we relax the discrete sampling process using the Gumbel-Top-$K$ trick. Given logits $\boldsymbol{\phi}_i = [\phi_{i1}, \ldots, \phi_{iM}]$ forming a categorical distribution over the $M$ consecutive model weights within the $i$-th group, the probability of selecting the $j$-th weight is achieved through the $\operatorname{softmax}$ function, $\pi_{ij} = \exp(\phi_{ij}) / \sum_{k=1}^M \exp(\phi_{ik})$.
To sample a subset $\mathcal{S}_i$ without replacement from this distribution, we first perturb the logits with independent Gumbel noise to obtain random keys:
\begin{equation}
    \kappa_{ij} = \phi_{ij} + g_{ij}, \,\,\, g_{ij} \overset{\text{i.i.d}}{\sim} \operatorname{Gumbel}(0,1).
\end{equation}
We denote the adjusted keys for the $i$-th group at sampling step $k$ by $\boldsymbol{\alpha}_i^{(k)} = [\alpha_{i1}^{(k)}, \ldots, \alpha_{iM}^{(k)}]$.
The update rule is defined recursively. At each subsequent step $k$, the adjusted keys are updated by accumulating a correction term based on the previously selected probabilities as follows:
\begin{equation}
\begin{aligned}
    \boldsymbol{\alpha}_i^{(1)} &= [\kappa_{i1}, \ldots, \kappa_{iM}],
    \\
    \boldsymbol{\alpha}_i^{(k)} &= \boldsymbol{\alpha}_i^{(k-1)} + \log(1 - \boldsymbol{\mu}_i^{(k-1)}).
\end{aligned}
\end{equation}
Finally, a differentiable relaxed one-hot vector $\boldsymbol{\mu}_i^{(k)} = [\mu_{i1}^{(k)}, \ldots, \mu_{iM}^{(k)}]$, representing the selected item at the $k$-th sampling step, is achieved by applying a $\operatorname{softmax}$ operation to the adjusted keys with temperature $\tau$ > 0:
\begin{equation}
    \mu_{ij}^{(k)} = \frac{\exp(\alpha_{ij}^{(k)} / \tau)}{\sum_{k=1}^M \exp(\alpha_{ik}^{(k)} / \tau)}.
\end{equation}
After $N$ sampling steps, a set of soft one-hot vectors representing selected weights is obtained. By summing up these vectors, a continuous relaxation of the $N$-hot pruning mask can be constructed, enabling gradient-based optimization of the objective.

\subsubsection{Temperature Annealing}
Having previously defined the temperature hyperparameter $\tau$ to control the sharpness of one-hot approximations, we now introduce a separate hyperparameter $\lambda$ to explicitly govern the level of stochasticity in the sampling process. This is realized by applying the Gumbel-Top-$K$ trick to the scaled logits, $\overline{\boldsymbol{\phi}}_i = \boldsymbol{\phi}_i / \lambda$, such that larger values of $\lambda$ induce higher sampling randomness.
In our experiments, we implement annealing schedules for both $\tau$ and $\lambda$ to progressively guide the mask learning process, starting with high randomness to encourage broad exploration and gradually converging to a small set of high-confidence solutions toward the end of training. Specifically, we adopt exponential annealing schedules: at the $t$-th training step, the temperatures are:
\begin{equation}
\label{eq:annealing}
\begin{aligned}
    \tau_t &= \max \left\{ \tau_\text{end}, \tau_\text{init} \times \left( \frac{\tau_\text{end}}{ \tau_\text{init}} \right) ^\frac{t}{T_\text{anneal}} \right\} ,\\
    % \tau_t &= \tau_\text{init} + \frac{t}{T} \left(\tau_\text{end} - \tau_\text{init}\right), % \times \left(1 - \frac{t}{T}\right) + \tau_\text{end} \times \frac{t}{T},
    \\
    \lambda_t &= \max \left\{ \lambda_\text{end}, \lambda_\text{init} \times \left( \frac{\lambda_\text{end}}{ \lambda_\text{init}} \right) ^\frac{t}{T_\text{anneal}} \right\},
    % \lambda_t &= \lambda_\text{init} \times \left(1 - \frac{t}{T}\right) + \lambda_\text{end} \times \frac{t}{T},
\end{aligned}         
\end{equation}
with $T_\text{anneal}$ being the number of annealing steps.

\section{Experiment}
\begin{table*}[!t]
\centering
\begin{adjustbox}{max width=\textwidth}
% \begin{tabular}{l|c|cccccc|c}
\begin{tabular}{lccccccccc}
\toprule
\textbf{Method}                   & \textbf{W/U} & \textbf{ARC-C} & \textbf{ARC-E} & \textbf{BoolQ} & \textbf{HellaS.} & \textbf{PIQA}  & \textbf{RACE}  & \textbf{SciQ}  & \textbf{Average} $\uparrow$ \\
\midrule
\rowcolor{headergray}
\textbf{Base Model: Qwen2.5-0.5B} & -               & 29.01          & 64.48          & 61.80          & 40.54            & 70.51          & 34.74          & 92.90          & 56.28            \\
Magnitude                         & \xmark          & 19.11          & 31.19          & 38.38          & 26.65            & 54.13          & 23.06          & 43.40          & 33.70            \\
Wanda                             & \xmark          & 19.80          & 42.85          & 56.82          & 28.93            & 59.03          & 26.70          & 83.20          & 45.33            \\
SparseGPT                         & \cmark          & 19.97          & 47.22          & 59.85          & 30.80            & 60.12          & 27.85          & 84.30          & 47.16            \\
ProxSparse                        & \xmark          & 21.42          & 46.63          & \underline{61.63}    & 32.00            & 61.63          & 29.09          & 79.80          & 47.46            \\
MaskLLM                           & \xmark          & \underline{22.16}    & \underline{54.02}    & \textbf{62.10} & \underline{32.28}      & \underline{63.41}    & \underline{30.19}    & \underline{85.90}    & \underline{50.01}      \\
\textbf{RoI (Ours)}              & \xmark          & \textbf{23.63} & \textbf{54.97} & 60.95          & \textbf{32.34}   & \textbf{63.82} & \textbf{31.00} & \textbf{87.20} & \textbf{50.56}   \\
\midrule
\rowcolor{headergray}
\textbf{Base Model: Qwen2.5-1.5B} & -               & 41.55          & 75.21          & 72.48          & 50.18            & 75.52          & 36.65          & 94.20          & 63.68            \\
Magnitude                         & \xmark          & 20.48          & 32.74          & 39.42          & 27.20            & 57.07          & 24.59          & 60.10          & 37.37            \\
Wanda                             & \xmark          & 24.06          & 51.26          & \textbf{63.64} & 32.36            & 62.46          & 29.28          & 88.30          & 50.19            \\
SparseGPT                         & \cmark          & 28.33          & 56.69          & \underline{62.66}    & 34.54            & 65.07          & \textbf{33.49} & 89.00          & 52.83            \\
ProxSparse                        & \xmark          & 29.25          & 59.55          & 42.87          & \textbf{39.61}   & 67.25          & 32.73          & \underline{89.80}    & 51.58            \\
MaskLLM                           & \xmark          & \textbf{29.45} & \textbf{63.51} & 61.33          & 39.10            & \underline{69.07}    & \underline{32.51}    & \textbf{89.90} & \underline{54.98}      \\
\textbf{RoI (Ours)}              & \xmark          & \underline{29.37}    & \underline{63.01}    & 61.93          & \underline{39.29}      & \textbf{69.31} & 32.44          & \textbf{89.90} & \textbf{55.04}   \\
\midrule
\rowcolor{headergray}
\textbf{Base Model: Qwen2.5-3B}   & -               & 44.62          & 77.57          & 77.22          & 55.04            & 78.29          & 38.47          & 95.90          & 66.73            \\
Magnitude                         & \xmark          & 23.81          & 24.66          & 37.83          & 25.98            & 51.96          & 20.77          & 18.50          & 29.07            \\
Wanda                             & \xmark          & 27.13          & 60.02          & 62.26          & 35.88            & 67.14          & 34.35          & 91.00          & 53.97            \\
SparseGPT                         & \cmark          & 31.40          & 66.12          & 62.66          & 41.05            & 69.75          & 35.12          & \underline{92.50}    & 56.94            \\
ProxSparse                        & \xmark          & -              & -              & -              & -                & -              & -              & -              & -                \\
MaskLLM                           & \xmark          & \textbf{32.95} & \underline{66.42}    & \underline{65.11}    & \underline{43.18}      & \underline{70.44}    & \textbf{35.35} & \textbf{93.00} & \textbf{58.06}   \\
\textbf{RoI (Ours)}              & \xmark          & \underline{32.08}    & \textbf{66.71} & \textbf{65.26} & \textbf{43.30}   & \textbf{70.57} & \underline{35.22}    & \textbf{93.00} & \underline{58.02}      \\
\midrule
\rowcolor{headergray}
\textbf{Base Model: Qwen2.5-7B}   & -               & 48.21          & 80.43          & 84.65          & 59.98            & 78.78          & 41.91          & 96.60          & 70.08            \\
Magnitude                         & \xmark          & 24.40          & 38.34          & 64.16          & 30.05            & 55.60          & 23.25          & 74.30          & 44.30            \\
Wanda                             & \xmark          & 38.99          & 72.18          & 73.21          & 44.53            & 71.87          & 35.79          & 94.50          & 61.58            \\
SparseGPT                         & \cmark          & \textbf{39.65} & 73.54          & 75.83          & 45.89            & 73.38          & 36.67          & 94.80          & 62.82            \\
ProxSparse                        & \xmark          & 38.69          & \underline{74.46}    & 74.77          & 48.63            & 73.22          & 36.08          & \textbf{95.60} & 63.06            \\
MaskLLM                           & \xmark          & \underline{39.27}    & 73.68          & \underline{81.55}    & \textbf{49.12}   & \underline{74.88}    & \underline{37.46}    & \underline{95.10}    & \underline{64.44}      \\
\textbf{RoI (Ours)}              & \xmark          & 38.14          & \textbf{74.20} & \textbf{82.14} & \underline{49.04}      & \textbf{75.35} & \textbf{37.80} & \underline{95.10}    & \textbf{64.54}   \\
\bottomrule
\end{tabular}
\end{adjustbox}
\caption{Comparative evaluation of zero-shot accuracy across multiple benchmark datasets for various pruning methods on the Qwen2.5 model family at varying scales with 2:4 sparsity pattern. \textbf{Bold} and \underline{underlined} values indicate the highest and second-highest performance, respectively. The `W/U' column denotes whether weight updates are applied during pruning. ProxSparse fails to converge for the 3B model.}
\label{tab:acc_2:4}
\end{table*}

\subsection{Experimental Setting}
\textbf{LLM backbones}: The proposed method is evaluated 
on the Qwen2.5 model family~\cite{abs-2412-15115} at multiple scales, from 0.5B to 7B, to assess its stability and scalability under semi-structured pruning. All main experiments adopt a 2:4 sparsity pattern, compatible with NVIDIA hardware acceleration. Training is performed for 2,000 steps with a batch size of 256 and a sequence length of 4096, processing approximately 2B tokens in total. The training data is drawn from the Nemotron-CC-v2 corpus~\cite{nvidia2025nvidianemotronnano2}, a clean English dataset suitable for pre-training modern LLMs. We detail hyperparameters in Appendix~\ref{apx:hyper}.
\\
\textbf{Baselines}:
To assess the effectiveness of the proposed method, we select five representative semi-structured pruning baselines spanning both classical and modern approaches: Magnitude~\cite{HanPTD15}, Wanda~\cite{Sun0BK24}, SparseGPT~\cite{FrantarA23}, ProxSparse~\cite{LiuSJPHS0K25}, and MaskLLM~\cite{FangYMHPK0W24}. For ProxSparse, we use the same amount of training data as RoI and MaskLLM to ensure fair comparisons, even though ProxSparse typically requires less data. Performance across different data scales is reported in Section~\ref{sec:scalability}, and while detailed methodological descriptions of baselines are provided in Appendix~\ref{apx:baseline}.
\subsection{Main Results}
% ACC
Following prior work, we begin by evaluating zero-shot accuracies on a broad suite of standard benchmark datasets (Table~\ref{tab:acc_2:4}). Across all model scales and evaluation tasks, sparsity-mask learning approaches (ProxSparse, MaskLLM, and RoI) consistently outperform more saliency-based pruning baselines, with only minor exceptions.
Among these methods, approaches that explicitly learn the pruning mask (MaskLLM and RoI) via variational optimization substantially exceed deterministic regularization-based learning (ProxSparse), improving accuracy by 2.56\footnote{This average is computed from the differences between the `\texttt{Average}' accuracies of MaskLLM/RoI and ProxSparse across all model scales reported in Table~\ref{tab:acc_2:4} (excluding Qwen2.5-3B where ProxSparse fails to converge).} percentage points on average across model scales and benchmarks. Despite RoI and MaskLLM both achieving comparable results (e.g. 64.54\% vs.\ 64.44\% average accuracy for the 7B model), RoI incurs substantially lower pruning overhead. Specifically, for 2:4 sparsity, RoI reduces the number of trainable parameters by approximately 33\%, using only 1$\times$ the original model's parameter count, compared to 1.5$\times$ for MaskLLM (illustrated in Table~\ref{tab:num-parameters}).

% PPL
\begin{table}[ht]
\centering
\begin{adjustbox}{max width=\linewidth}
\begin{tabular}{lcccc}
\toprule
& \multicolumn{4}{c}{\textbf{Model Sizes}} \\ \cmidrule(lr){2-5}
\textbf{Method}      & \textbf{0.5B}  & \textbf{1.5B}  & \textbf{3B}    & \textbf{7B}    \\
\midrule
\rowcolor{headergray}
\textbf{w/o Pruning} & 18.45          & 12.59          & 12.32          & 9.03           \\
Magnitude            & 3.0E+4& 5.2E+3& 1.6E+6& 6.5E+3\\
Wanda                & 159.05         & 67.97          & 37.31          & 20.87          \\
SparseGPT            & 73.78          & 38.95          & 22.82          & 16.44          \\
ProxSparse           & 146.42         & 28.68          & -          & \underline{14.91}    \\
MaskLLM              & \underline{35.16}    & \underline{21.90}    & \underline{19.85}    & \underline{14.91}    \\
\midrule
\textbf{RoI (Ours)} & \textbf{32.63} & \textbf{21.27} & \textbf{19.37} & \textbf{14.89} \\
\bottomrule
\end{tabular}
\end{adjustbox}
\caption{Perplexity (PPL) scores on WikiText-2 benchmark dataset and training flops (FLOPs) across backbone models with 2:4 sparsity pattern setting. \textbf{Bold} and \underline{underlined} values indicate the best and second-best performance, respectively.}
\label{tab:perplexity_PPL_2:4}
\end{table}

Additionally, Table~\ref{tab:perplexity_PPL_2:4} reports perplexity (PPL) values on the WikiText-2 dataset, further highlighting the advantages of our proposed method:
\textbf{(i) Effectiveness of Differentiable Subset Sampling:}
RoI consistently outperforms all competing baselines at every model scale, indicating that its differentiable subset sampling mechanism reliably discovers high-quality sparsity patterns that better preserve model performance after pruning.
\textbf{(ii) Scalability across Model Sizes:}
RoI inherits the strong generalization properties of MaskLLM and delivers consistent improvements as the model scale increases, enabled by a closely related formulation and optimization strategy. RoI maintains its advantage over simpler methods from small to large models: achieving 14.89 perplexity at 7B parameters, compared to 20.87 for Wanda and 16.44 for SparseGPT. This stable scaling behavior contrasts sharply with Magnitude pruning - which attains its best perplexity at 1.5B parameters rather than at 7B, further emphasizing RoI’s ability to maintain reliable performance as model capacity grows. In contrast, ProxSparse exhibits unstable optimization behavior at larger scales. In particular, when applied to the Qwen2.5-3B model, ProxSparse fails to reliably converge, preventing it from achieving competitive performance.
\textbf{(iii) Robustness of Learnable Sparsity-Mask Approaches:}
Magnitude pruning performs poorly (consistently yielding perplexity values exceeding 5,000), and other naive baselines fall far behind methods that explicitly learn the pruning mask, reaffirming that simple pruning strategies substantially impair language modeling performance. The strong results of RoI and MaskLLM underscore the importance of principled, learnable pruning mechanisms.

% Training Flops
RoI and MaskLLM adopt similar problem formulations and largely comparable optimization procedures. Although RoI overall outperforms MaskLLM in both zero-shot accuracy and perplexity, the raw performance gap remains relatively modest. The primary advantage of RoI instead lies in its significantly lower training cost, as illustrated in Table~\ref{tab:num-parameters}, which compares the number of trainable parameters required by each method.
In particular, the training cost of MaskLLM grows rapidly with model size. Even under 2:4 sparsity, MaskLLM requires 1.5$\times$ more trainable parameters than RoI, a difference that may appear moderate in ratio but translates into a large absolute increase at scale. For the Qwen2.5-7B model, achieving 2:4 sparsity with MaskLLM requires $\approx$9.75B additional trainable parameters.

\begin{table}[t]
\centering
\begin{adjustbox}{max width=\linewidth}
\begin{tabular}{lcccc}
\toprule
& \multicolumn{4}{c}{\textbf{Model Sizes}} \\ \cmidrule(lr){2-5}
\textbf{Method} & \textbf{0.5B} & \textbf{1.5B} & \textbf{3B} & \textbf{7B} \\
\midrule
\rowcolor{headergray}
\multicolumn{5}{c}{\textbf{2:4 Sparsity}} \\
ProxSparse & 0.4B (1$\times$)& 1.3B (1$\times$)& 2.7B (1$\times$)& 6.5B (1$\times$)\\
MaskLLM   & 0.6B (1.5$\times$)& 2.0B (1.5$\times$)& 4.0B (1.5$\times$)& 9.8B (1.5$\times$)\\
\textbf{RoI (Ours)}      & 0.4B (1$\times$)& 1.3B (1$\times$)& 2.7B (1$\times$)& 6.5B (1$\times$)\\
\midrule
\rowcolor{headergray}
\multicolumn{5}{c}{\textbf{2:8 Sparsity}} \\
MaskLLM   & 1.4B (3.5$\times$)& 4.6B (3.5$\times$)& 9.5B (3.5$\times$)& 22.8B (3.5$\times$)\\
\textbf{RoI (Ours)}      & 0.4B (1$\times$)& 1.3B (1$\times$)& 2.7B (1$\times$)& 6.5B (1$\times$)\\
\bottomrule
\end{tabular}
\end{adjustbox}
\caption{Number of trainable parameters (in billions) under different sparsity patterns and model sizes. Values in parentheses indicate the ratio relative to the original dense model parameters (excluding embeddings). For mask-learning methods, only mask parameters are trainable while the underlying model weights remain frozen.}
\label{tab:num-parameters}
\end{table}

% ====================== Detailed Analysis ====================== %
\subsection{Detailed Analysis}

\subsubsection{2:8 Sparsity as a Failure Regime}

To further evaluate the generalization of RoI, we conduct a stress test under the 2:8 configuration. Because 2:8 sparsity is not broadly supported by mainstream accelerators, this experiment is intended for scalability evaluation rather than deployment. This aggressive sparsity exposes the challenges faced by semi-structured pruning methods, particularly those that do not learn pruning masks. % Under this regime, we directly compare RoI with MaskLLM to highlight pruning efficiency and effectiveness.

As shown in Table~\ref{tab:perplexity_PPL_2:8}, only methods that learn pruning masks achieve reasonable perplexity values, while saliency-based baselines collapse, producing perplexities exceeding 100. This behavior is congruent with the 2:4 setting and highlight the critical necessity of learnable mask parameterization. Accuracy results under 2:8 sparsity also exhibit the same trend and are deferred to Appendix~\ref{apx:acc_2:8}.

\begin{table}[t]
\centering
\begin{adjustbox}{max width=\linewidth}
\begin{tabular}{lcccc}
\toprule
& \multicolumn{4}{c}{\textbf{Model Sizes}} \\ \cmidrule(lr){2-5}
\textbf{Method}      & \textbf{0.5B}  & \textbf{1.5B}  & \textbf{3B}    & \textbf{7B}    \\
\midrule
\rowcolor{headergray}
\textbf{w/o Pruning} & 18.45          & 12.59          & 12.32          & 9.03           \\
Magnitude            & 1.5E+7& 1.0E+7& 6.6E+7& infinity\footnotemark       \\
Wanda                & 8.1E+5& 1.0E+6& 3.2E+6& 4.7E+4\\
SparseGPT            & 1.1E+4& 4.0E+4& 257.22         & 158.70         \\
MaskLLM              & \underline{63.11}    & \underline{38.14}    & \underline{38.89}    & \underline{34.72}    \\
\midrule
\textbf{RoI (Ours)} & \textbf{61.89} & \textbf{36.87} & \textbf{37.72} & \textbf{32.33} \\
\bottomrule
\end{tabular}
\end{adjustbox}
\caption{WikiText-2 Perplexity (PPL) scores across backbone models with 2:8 sparsity pattern setting. \textbf{Bold} and \underline{underlined} values indicate the best and second-best performance, respectively. ProxSparse is excluded due to its inability to learn 2:8 sparsity patterns.}
\label{tab:perplexity_PPL_2:8}
\end{table}

In conjunction with Table~\ref{tab:num-parameters}, we again observe the rapid growth in training cost of MaskLLM as model size increases. Moreover, the parameter-requirement gap between RoI and MaskLLM widens substantially under more aggressive sparsity. For the 7B model at 2:8 sparsity, MaskLLM requires more than twice the auxiliary parameters compared to the 2:4 setting. This corresponds to $\approx$22.75B additional parameters relative to RoI, resulting in a 3.5$\times$ increase in parameters participating in memory allocation, backpropagation, optimizer and scheduler states, while only achieving perplexity levels comparable to RoI.

These results demonstrate that differentiable subset sampling is sufficient to sustain language modeling capability under extreme sparsity, while requiring only 28.57\% of the trainable parameters of the categorical parameterization. This establishes RoI as the minimal learnable-mask formulation capable of minimizing accuracy loss.

\footnotetext{Undefined perplexity due to zero-probability predictions.}

\subsubsection{Perplexity over Training Tokens}

Among mask-learning pruning approaches, deterministic regularization-based learning fundamentally limits how effectively a model can explore in the parameter space due to the bias-variance tradeoff. While regularization can accelerate early-stage optimization, it introduces a persistent bias that constrains convergence as training data grows. In contrast, directly learning sparsity masks in a variational manner avoids this limitation and is therefore intrinsically better suited to large-scale learning regimes. This difference is directly reflected in Figure~\ref{fig:perplexity-per-training-tokens}. During the early stage of training (until approximately 0.5B tokens), the regularization-based pruning baseline ProxSparse attains lower perplexity. However, this advantage is temporary: as additional training data is incorporated, the performance of ProxSparse rapidly plateaus. Meanwhile, RoI and MaskLLM continue to improve steadily and monotonically, consistently translating additional data into lower perplexity. Although RoI and MaskLLM optimize the same variational objective, RoI exhibits stronger scalability in practice; see Appendix~\ref{apx:convergence-explanation} for a detailed analysis. These results provide clear empirical evidence that learning masks directly is critical for unlocking improvements under large-scale learning settings and achieving superior asymptotic performance.

\label{sec:scalability}
\begin{figure}[!t]
    \centering
    \includegraphics[width=\columnwidth]{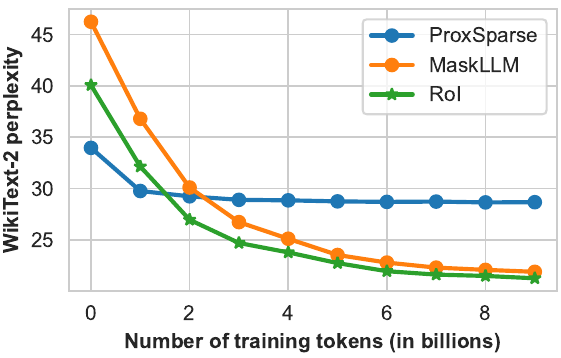}
    \caption{RoI demonstrates better scalability than MaskLLM and ProxSparse, with WikiText-2 perplexity (on the Qwen2.5-1.5B model) improving more consistently as the number of training tokens increases.}
    \label{fig:perplexity-per-training-tokens}
\end{figure}

\section{Conclusion}
This work introduced RoI, a novel semi-structured pruning technique that employs differentiable subset sampling to efficiently derive N:M sparsity masks. By substantially reducing the number of trainable parameters and associated memory overhead, RoI enables effective compression of LLMs while maintaining strong downstream performance. Across Qwen2.5 models ranging 0.5 to 7B parameters, RoI consistently achieves lower perplexity on WikiText-2 compared to existing pruning approaches, and maintains competitive zero-shot accuracy on a diverse set of benchmarks. Moreover, RoI demonstrates enhanced data efficiency and favorable scalability as the amount of calibration data increases. Overall, these results position RoI as a practical and scalable solution for semi-structured pruning of LLMs, offering a strong balancing between compression efficiency and performance preservation for resource-constrained deployment environments.

\section*{Limitations}

Despite the promising performance and efficiency demonstrated by RoI, several limitations remain: (i)
First, the deployment of semi-structured sparsity is inherently hardware-dependent. At present, substantial throughput gains are realized only on select platforms (e.g., AMD ROCm and certain NVIDIA Ampere and Hopper GPUs) where 2:4 structured sparsity is natively supported and accelerated at the kernel level. Although RoI can, in principle, be extended to arbitrary $N\!:\!M$ sparsity patterns, its practical utility is constrained by the absence of hardware kernels and vendor-optimized libraries for ratios other than 2:4. On accelerators or CPUs lacking such specialized support, pruning yields only marginal reductions in memory footprint and fails to deliver meaningful inference speed up. This hardware dependency poses a significant challenge for widespread adoption in heterogeneous production environments, where deployment targets may vary; (ii)
Second, the current evaluation focuses exclusively on English-centric, dense models and a limited set of standard NLP benchmarks. Future research should investigate the applicability of RoI to multilingual LLMs, different architectures like Mixture-of-Experts, larger-scale models, and domain-specific tasks (e.g., code generation, reasoning-intensive applications) to assess its generalization and scalability comprehensively.

\iffalse
\section*{Acknowledgments}

% This document has been adapted
% by Steven Bethard, Ryan Cotterell and Rui Yan
% from the instructions for earlier ACL and NAACL proceedings, including those for
% ACL 2019 by Douwe Kiela and Ivan Vuli\'{c},
% NAACL 2019 by Stephanie Lukin and Alla Roskovskaya,
% ACL 2018 by Shay Cohen, Kevin Gimpel, and Wei Lu,
% NAACL 2018 by Margaret Mitchell and Stephanie Lukin,
% Bib\TeX{} suggestions for (NA)ACL 2017/2018 from Jason Eisner,
% ACL 2017 by Dan Gildea and Min-Yen Kan,
% NAACL 2017 by Margaret Mitchell,
% ACL 2012 by Maggie Li and Michael White,
% ACL 2010 by Jing-Shin Chang and Philipp Koehn,
% ACL 2008 by Johanna D. Moore, Simone Teufel, James Allan, and Sadaoki Furui,
% ACL 2005 by Hwee Tou Ng and Kemal Oflazer,
% ACL 2002 by Eugene Charniak and Dekang Lin,
% and earlier ACL and EACL formats written by several people, including
% John Chen, Henry S. Thompson and Donald Walker.
% Additional elements were taken from the formatting instructions of the \emph{International Joint Conference on Artificial Intelligence} and the \emph{Conference on Computer Vision and Pattern Recognition}.
\fi
\bibliography{custom}

\appendix

\section{Appendix}
\subsection{WRS Yields Equivalent Variational Objective}
% TODO: Relocate this proof section
\begin{theorem}
\label{thm:restricted-sampling}
    Let $\Pr(\mathbf{m}_i | \boldsymbol{\phi}_i)$ be the target distribution of each mask $\mathbf{m}_i$ as defined in Equation~\ref{eq:full_dist_definition}. The expected loss obtained by sampling each mask from its target distribution is equal to the expected loss obtained when each mask is parameterized as a sum of elements of an ordered subset $\mathcal{S}_i$ sampled from the corresponding restricted distribution $\Pr_\text{\textnormal{WRS}}(\mathcal{S}_i | \boldsymbol{\phi}_i)$.
\end{theorem}

We show that:
\begin{equation}
    \mathbb{E}_{\Pr(\mathbf{m} | \boldsymbol{\phi})} [f\!\left(\mathbf{m}\right)] = \mathbb{E}_{\Pr_\text{WRS}(\mathcal{S} | \boldsymbol{\phi})} \left[ f\!\left( \sum_{\boldsymbol{\mu} \in \mathcal{S}} \boldsymbol{\mu} \right) \right],
\end{equation}
where $f(\mathbf{m})$ is an objective function of $\mathbf{m}$, 
and $\Pr(\mathbf{m} | \boldsymbol{\phi}) = \sum_{\mathcal{S}_\mathbf{m}} \Pr_\text{WRS}(\mathcal{S}_\mathbf{m} | \boldsymbol{\phi})$
with $\mathcal{S}_{\mathbf{m}}$ ranging over all sets whose elements sum to $\mathbf{m}$.

\begin{proof}
Given any set of binary masks $\mathcal{M}$ satisfying $N$:$M$ sparsity, by definition of $\Pr(\mathbf{m}\mid\boldsymbol{\phi})$:
\begin{equation*}
\begin{aligned}
    &\mathbb{E}_{\Pr(\mathbf{m}\mid\boldsymbol{\phi})}[f(\mathbf{m})]
    =
    \sum_{\mathbf{m}\in\mathcal{M}} 
        \Pr(\mathbf{m}\mid\boldsymbol{\phi})\, f(\mathbf{m}) \\
    &=
    \sum_{\mathbf{m}\in\mathcal{M}}
        \left(
        \sum_{\mathcal{S}_{\mathbf{m}}}
                \Pr_{\mathrm{WRS}}(\mathcal{S}_{\mathbf{m}}\mid\boldsymbol{\phi})
        \right)
        f\!\left(\sum_{\boldsymbol{\mu}\in\mathcal{S}_{\mathbf{m}}}\boldsymbol{\mu}\right).
\end{aligned}
\end{equation*}
Reordering the sums yields:
\begin{equation*}
\begin{aligned}
    &\sum_{\mathbf{m}\in\mathcal{M}}
    \sum_{\mathcal{S}_{\mathbf{m}}}
        \Pr_{\mathrm{WRS}}(\mathcal{S}_{\mathbf{m}}\mid\boldsymbol{\phi})\,
        f\!\left(\sum_{\boldsymbol{\mu}\in\mathcal{S}_{\mathbf{m}}}\boldsymbol{\mu}\right) \\
    &=
    \sum_{\mathcal{S}}
        \Pr_{\mathrm{WRS}}(\mathcal{S}\mid\boldsymbol{\phi})\,
        f\!\left(\sum_{\boldsymbol{\mu}\in\mathcal{S}}\boldsymbol{\mu}\right),
\end{aligned}
\end{equation*}
which equals:
\begin{equation*}
\mathbb{E}_{\Pr_{\mathrm{WRS}}(\mathcal{S}\mid\boldsymbol{\phi})}
\!\left[
f\!\left(\sum_{\boldsymbol{\mu}\in\mathcal{S}}\boldsymbol{\mu}\right)
\right].
\end{equation*}
\end{proof}
% Intuitively, this follows because 
 
\subsection{Baseline Methods}
\label{apx:baseline}
i)~\textbf{Magnitude} is a simple, data-free method that removes parameters with the smallest absolute values. While this method is easy to implement, it often yields subpar results due to the limitations of parameter sensitivity and model dynamics;
ii)~\textbf{Wanda}~\cite{Sun0BK24} combines parameter magnitudes with activation statistics at each layer, achieving better performance than pure magnitude pruning, especially under higher sparsity constraints, while maintaining favorable computational efficiency;
iii)~\textbf{SparseGPT}~\cite{FrantarA23} incorporates activation outputs and Hessian information to estimate parameter importance, followed by parameter updates to further reduce output error. This method can yield higher accuracy but is more computationally demanding;
iv)~\textbf{ProxSparse}~\cite{LiuSJPHS0K25} leverages regularization to quickly find sparsity masks under very constrained data and computational settings. Owing to its deterministic optimization procedure, ProxSparse typically converges rapidly in the early stages of training. However, this early convergence can limit its ability to further explore the solution space, thereby restricting its capacity to fully exploit additional training data or extended optimization; and
v)~\textbf{MaskLLM}~\cite{FangYMHPK0W24} is quite similar to the proposed method in this study, by learning pruning masks with minimizing calibration loss under an $N$:$M$ sparsity constraint, modeled via a multinomial distribution. It delivers strong performance across benchmarks but suffers from high computational cost. 

\subsection{Evaluation Metrics and Benchmark Datasets}
\label{apx:metrics-and-datasets}
Following previous works in this research field, three automated metrics are considered for the evaluation, including both quantitative and qualitative metrics to capture the full impact of pruning: i) \textit{Task Accuracy (ACC)}: on common NLP tasks such as question answering in reading comprehension, mathematics, and science. These tasks are typically assessed in zero-shot or few-shot settings using benchmark datasets; \textit{Perplexity (PPL)}: is a standard metric for assessing language model quality. 
It measures how well the model predicts the next word in a sequence, with lower values indicating better predictive performance.

\begin{table}[ht]
\centering
\begin{adjustbox}{max width=\columnwidth}
\begin{tabular}{lcl}
\toprule
% \rowcolor{darkgray}
\textbf{Dataset} & \textbf{Questions} & \textbf{Task Type}  \\
\midrule
ARC-Easy & 2,376	&Multiple-choice science	 \\
ARC-Challenge  &	1,172 &	Multiple-choice science\\ 
BoolQ & 3,270 &  Yes/no questions \\
HellaSwag &	10,042	&Sentence completion		\\
PIQA &	1,838	&Physical interaction QA	\\
RACE &	1,045	&Multiple-choice comprehension \\
SciQ &	1,000	&Multiple-choice science\\
\bottomrule
\end{tabular}
\end{adjustbox}
\caption{Dataset statistics for zero-shot evaluation. We report the number of questions and the corresponding task type for each benchmark.}
\label{tab:data}
\end{table}

The benchmark datasets used to assess the effectiveness of pruning methods include WikiText-2~\cite{MerityX0S17} for perplexity evaluation and a range of NLP benchmark datasets for zero-shot evaluation, which cover diverse task types and reasoning requirements, including ARC~\cite{abs-1803-05457}, BoolQ~\cite{ClarkLCK0T19}, HellaSwag~\cite{ZellersHBFC19}, PIQA~\cite{BiskZLGC20}, SciQ~\cite{WelblLG17}, and RACE~\cite{LaiXLYH17}. All evaluations are conducted using the LM-Evaluation-Harness toolkit~\cite{eval-harness}.  Table~\ref{tab:data} provides a comprehensive summary of the datasets used for zero-shot evaluation across multiple tasks. These datasets span a range of domains, including commonsense reasoning, science question answering, and reading comprehension, thereby ensuring a rigorous and diverse assessment of pruning performance.

\subsection{Hyperparameter Setting}
\label{apx:hyper}

The hyperparameters used for training RoI are listed in Table~\ref{tab:hyper}. These settings were carefully chosen to balance convergence stability and computational efficiency across all evaluated models.

\begin{table}[ht]
\centering
\begin{adjustbox}{max width=\columnwidth}
\begin{tabular}{lc}
\toprule
% \rowcolor{darkgray}
\textbf{Parameter} & \textbf{Values} \\
\midrule
Initialization distribution &$\mathcal{N}(0,0.01)$		 \\

Gumbel-Softmax temperature &$\tau=1.0 \rightarrow 0.05$	 \\

Sampling temperature &$\lambda = 1.0 \rightarrow 0.002$		 \\

Weight decay & $0.05$	 \\

Learning rate &$ 10^{-3} \rightarrow 10^{-4} $		 \\

AdamW parameters & $\beta_1 = 0.9, \beta_2=0.95$ \\

Batch size & 256	 \\

Sequence length & 4096	 \\
Training steps & 2000	 \\
Annealing steps & $T_\text{anneal} = 1500$ \\
\bottomrule
\end{tabular}
\end{adjustbox}
\caption{Hyperparameter configuration used in training.}
\label{tab:hyper}
\end{table}

Specifically, model weights remain frozen during training. The variational distribution is initialized from a standard normal ($\mu = 0.0$, $\sigma = 0.01$), and a simulated annealing process gradually reduces randomness. For 1500 first training steps, temperatures $\tau$ and $\lambda$ exponentially decay from 1.0 to 0.05 and from 1.0 to 0.002, respectively. Optimization uses AdamW with a learning rate decaying from $1 \times 10^{-3}$ to $1 \times 10^{-4}$, weight decay value of 0.05, and $\beta_1$ = 0.9, $\beta_2$ = 0.95.

\subsection{Annealing Schedule Ablation}
\label{apx:annealing-ablation}

We further study the effect of the annealing schedule for the sampling temperature $\lambda$ and the Gumbel-Softmax temperature $\tau$. Table~\ref{tab:annealing-ablation} compares exponential and linear schedules with different final values $(\lambda_\text{min}, \tau_\text{min})$ on the Qwen2.5-1.5B model. Overall, exponential annealing consistently outperforms linear annealing in both perplexity and average zero-shot accuracy. This suggests that reducing the temperatures more aggressively in the early stage, followed by smaller refinements near the end of training, provides a better balance between exploration and convergence. Among the exponential schedules, $(0.001, 0.05)$ yields the best average zero-shot accuracy, while $(0.002, 0.05)$ gives the lowest perplexity. We therefore use the exponential schedule with $(\lambda_\text{min}, \tau_\text{min})=(0.002, 0.05)$ as the default setting, as it provides the strongest language modeling performance while maintaining competitive downstream accuracy.

\begin{table*}[ht]
\centering
\begin{adjustbox}{max width=\textwidth}
\begin{tabular}{lc|c|cccccccc}
\toprule
\textbf{Schedule} & \textbf{$(\lambda_\text{min},\tau_\text{min})$} & \textbf{PPL} $\downarrow$ & \textbf{ARC-C} & \textbf{ARC-E} & \textbf{BoolQ} & \textbf{HellaS.} & \textbf{PIQA} & \textbf{RACE} & \textbf{SciQ} & \textbf{Avg.} $\uparrow$ \\
\midrule
% \rowcolor{headergray}
% \multicolumn{10}{c}{\textbf{Exponential}} \\
\multirow{3}{*}{Exponential} & $(0.001, 0.05)$ & 21.57 & 28.50 & 63.47 & 62.75 & 40.06 & 70.67 & 34.26 & 89.00 & 55.53 \\
                             % & $(0.001, 0.01)$ & -     & -     & -     & -     & -     & -     & -     & -     & -     \\
                             & $(0.002, 0.05)$ & 21.27 & 29.37 & 63.01 & 61.93 & 39.29 & 69.31 & 32.44 & 89.90 & 55.04 \\
                             % & $(0.002, 0.01)$ & -     & -     & -     & -     & -     & -     & -     & -     & -     \\
                             & $(0.003, 0.05)$ & 22.29 & 29.18 & 61.99 & 63.85 & 39.24 & 68.55 & 33.68 & 88.30 & 54.97 \\
                             % & $(0.003, 0.01)$ & -     & -     & -     & -     & -     & -     & -     & -     & -     \\
\midrule
% \rowcolor{headergray}
% \multicolumn{10}{c}{\textbf{Linear}} \\
\multirow{3}{*}{Linear} & $(0.001, 0.05)$ & 26.80 & 28.18 & 60.85 & 61.70 & 38.42 & 68.21 & 31.83 & 87.70 & 53.84 \\
                        % & $(0.001, 0.01)$ & 21.27 & 29.37 & 63.01 & 61.93 & 39.29 & 69.31 & 32.44 & 89.90 & 55.04 \\
                        & $(0.002, 0.05)$ & 27.04 & 27.59 & 60.31 & 62.11 & 37.64 & 67.53 & 32.29 & 87.10 & 53.51 \\
                        % & $(0.002, 0.01)$ & 21.27 & 29.37 & 63.01 & 61.93 & 39.29 & 69.31 & 32.44 & 89.90 & 55.04 \\
                        & $(0.003, 0.05)$ & 27.51 & 28.51 & 61.99 & 61.25 & 37.07 & 67.12 & 32.29 & 87.10 & 53.62 \\
                        % & $(0.003, 0.01)$ & 21.27 & 29.37 & 63.01 & 61.93 & 39.29 & 69.31 & 32.44 & 89.90 & 55.04 \\
\bottomrule
\end{tabular}
\end{adjustbox}
\caption{Ablation results for different annealing schedules and final temperature values on the Qwen2.5-1.5B model. PPL denotes WikiText-2 perplexity, and Avg. denotes the average zero-shot accuracy over the downstream tasks.}
\label{tab:annealing-ablation}
\end{table*}

\subsection{Training Time Comparison}
\label{apx:training-time}

To complement the parameter-efficiency analysis in Table~\ref{tab:num-parameters}, we report the end-to-end training time of learnable semi-structured pruning methods across model scales in Table~\ref{tab:training-time}.

\begin{table}[H]
\centering
\begin{adjustbox}{max width=\columnwidth}
\begin{tabular}{lcccc}
\toprule
& \multicolumn{4}{c}{\textbf{Model Sizes}} \\ \cmidrule(lr){2-5}
\textbf{Method} & \textbf{0.5B} & \textbf{1.5B} & \textbf{3B} & \textbf{7B} \\
\midrule
ProxSparse & 12.05 & 16.62 & 32.66 & 52.41 \\
MaskLLM & 9.00 & 16.93 & 30.95 & 55.39 \\
\textbf{RoI (Ours)} & \textbf{7.12} & \textbf{13.17} & \textbf{24.36} & \textbf{46.10} \\
\bottomrule
\end{tabular}
\end{adjustbox}
\caption{Reported training time (in GPU hours) across various sizes of the Qwen 2.5 model family for learnable semi-structured pruning methods. Lower values indicate faster training.}
\label{tab:training-time}
\end{table}

All methods are trained under the same 2:4 sparsity setting on NVIDIA B300 GPUs and use the same training budget described in Appendix~\ref{apx:hyper}. The results show that RoI consistently requires less training time than both ProxSparse and MaskLLM across all evaluated model sizes. This advantage becomes more pronounced as the model scale increases: on the 7B model, RoI reduces training time from 52.41 hours to 46.10 hours compared with ProxSparse, and from 55.39 hours to 46.10 hours compared with MaskLLM. These results support the claim that the compact subset-sampling parameterization improves practical training efficiency in addition to reducing trainable parameters.

\subsection{Optimization Efficiency}
\label{apx:convergence-explanation}

Although RoI and MaskLLM optimize equivalent variational objectives over $N$:$M$ masks, they use different parameterizations of the same mask distribution. MaskLLM assigns a separate logit to each feasible mask configuration, requiring $C=\binom{M}{N}$ logits per pruning group. In contrast, RoI assigns one logit to each item in the group and induces a distribution over masks through weighted sampling without replacement, requiring only $M$ logits. This distinction does not change the target objective, but it changes the geometry of the optimization problem and the efficiency of gradient propagation.

Let $\mathcal{M}$ denote the set of all feasible $N$-hot masks in a group. A full categorical mask learner parameterizes
\begin{equation}
    \Pr_\text{cat}(\mathbf{m} | \boldsymbol{\alpha}) = \frac{\exp(\alpha_{\mathbf{m}})}{\sum_{\mathbf{m'}\in\mathcal{M}} \exp(\alpha_{\mathbf{m'}})},\ \mathbf{m}\in\mathcal{M},
\end{equation}
where each configuration has its own independent parameter $\alpha_{\mathbf{m}}$. Therefore, a gradient update observed for one sampled mask primarily updates that individual configuration. Information about the usefulness of a particular retained weight is not directly shared with other masks that contain the same weight. Since each weight participates in many feasible masks, this representation can require repeated observations of many related configurations before the optimizer consistently increases the probability of important weights.

RoI instead parameterizes the same decision through item-level logits $\boldsymbol{\phi}=[\phi_1,\ldots,\phi_M]$. The probability of an unordered mask $\mathbf{m}$ is induced by summing over WRS orderings as in Equation~\ref{eq:full_dist_definition}. Consequently, the gradient with respect to an item logit aggregates contributions from all sampled masks containing that item:
\begin{equation}
\begin{aligned}
    &\frac{\partial}{\partial \phi_j}
    \mathbb{E}_{\Pr_\text{RoI}(\mathbf{m} | \boldsymbol{\phi})}
    [\mathcal{L}(\mathbf{m})] \\
    &=
    \sum_{\mathbf{m}\in\mathcal{M}}
    \mathcal{L}(\mathbf{m})
    \frac{\partial \Pr_\text{RoI}(\mathbf{m} | \boldsymbol{\phi})}{\partial \phi_j}.
\end{aligned}
\end{equation}
Because $\phi_j$ influences every mask in which item $j$ can be selected, each update transfers evidence across a combinatorial family of masks rather than only a single categorical state. This parameter sharing provides a useful inductive bias for pruning: if a weight is consistently important, RoI can increase its selection probability across many masks simultaneously.

This explains the empirical convergence gap. The categorical MaskLLM parameterization is more expressive, but it is also higher-dimensional and less statistically efficient; optimization must distribute probability mass over $\binom{M}{N}$ configurations. RoI restricts the search to a structured WRS-induced family with only $M$ degrees of freedom, which reduces memory traffic, optimizer-state overhead, and Monte Carlo estimation noise. As a result, even though both methods correspond to equivalent variational objectives at the mask-distribution level, RoI follows a better-conditioned and more sample-efficient optimization path.

\subsection{Gumbel-Top-$K$ Algorithm}
The algorithm for Gumbel-Top-$K$ is illustrated in Algorithm~\ref{alg:algorithm}. Specifically, we provide a clear description of the Gumbel-Top-$K$ sampling procedure employed to enable differentiable mask learning. This formulation allows for efficient sampling of $K$ items without replacement while preserving differentiability for gradient-based optimization.
\
\begin{algorithm}[ht]
\caption{Gumbel-Top-$K$ Sampling Algorithm (Differentiable)}
\label{alg:algorithm}
\textbf{Input}: Set of candidates $\mathcal{X} = \{x_1, \ldots, x_L\}$ with corresponding logits $\boldsymbol{\phi} = [\phi_1, \ldots, \phi_L]$, number of samples $K$, temperature $\tau > 0$\\
\textbf{Output}: Soft $K$-hot selection vector $\mathbf{S} \in \mathbb{R}^N$
\begin{algorithmic}[1] %[1] enables line numbers
\FOR{$i \gets 1$ to $N$}
    \STATE $u_i \sim \operatorname{Uniform}(0, 1)$
    \STATE $g_i \gets -\log(-\log(u_i))$ \hfill \COMMENT{Gumbel noise}
    \STATE $\kappa_i \gets \phi_i + g_i$ 
    \hfill \COMMENT{Compute perturbed key}
\ENDFOR
\STATE $\boldsymbol{\alpha}^{(1)} \gets [\kappa_1, \ldots, \kappa_N]$
\FOR{$k \gets 1$ to $K$}
    \STATE $\boldsymbol{\mu}^{(k)} \gets \operatorname{softmax}\bigl(\boldsymbol{\alpha}^{(k)} / \tau\bigr)$
    \STATE $\boldsymbol{\alpha}^{(k+1)} \gets \boldsymbol{\alpha}^{(k)} + \log\bigl(1 - \boldsymbol{\mu}^{(k)}\bigr)$
\ENDFOR
\STATE $\mathbf{S} \gets \sum_{k=1}^{K} \boldsymbol{\mu}^{(k)}$ \hfill \COMMENT{Soft $K$-hot vector}
\STATE \textbf{return} $\mathbf{S}$
\end{algorithmic}
\end{algorithm}

\subsection{Comparison with Straight-Through Gumbel-Top-$K$}
We further examined whether adopting a straight-through (ST) Gumbel-Top-$K$ estimator benefits pruning performance. In this variant, the forward pass generates discrete masks by directly applying $\operatorname{argtop}_K$ over Gumbel-perturbed logits, while the backward pass propagates gradients through the continuous Gumbel-Softmax relaxation. This strategy enforces discretization earlier in training, which should be able to improve mask interpretability. However, our empirical results in Table~\ref{tab:ste} show that ST Gumbel-Top-$K$ leads to slightly inferior performance compared to the pure soft relaxation. 

\begin{table}[ht]
    \centering
    \begin{adjustbox}{max width=\columnwidth}
    \begin{tabular}{lcccc}
        \toprule
        & \multicolumn{2}{c}{\textbf{Qwen2.5-1.5B}} & \multicolumn{2}{c}{\textbf{Qwen2.5-3B}} \\ \cmidrule(lr){2-3} \cmidrule(lr){4-5}
        \textbf{Metric} & w STE & w/o STE (Ours) & w STE & w/o STE (Ours) \\
        \midrule
        PPL (↓) & 30.58& \textbf{21.27}& 23.75& \textbf{19.37}\\
        Avg. Acc (↑) & 53.10& \textbf{55.04}& 55.26& \textbf{58.02}\\
        \bottomrule
    \end{tabular}
    \end{adjustbox}
    \caption{Comparison of pruning results with 2:4 sparsity, both with and without ST Gumbel-Top-$K$ estimator (denoted as "w STE" and "w/o STE").}
    \label{tab:ste}
\end{table}

\begin{figure*}[!t]
    \centering
    \includegraphics[width=0.95\linewidth]{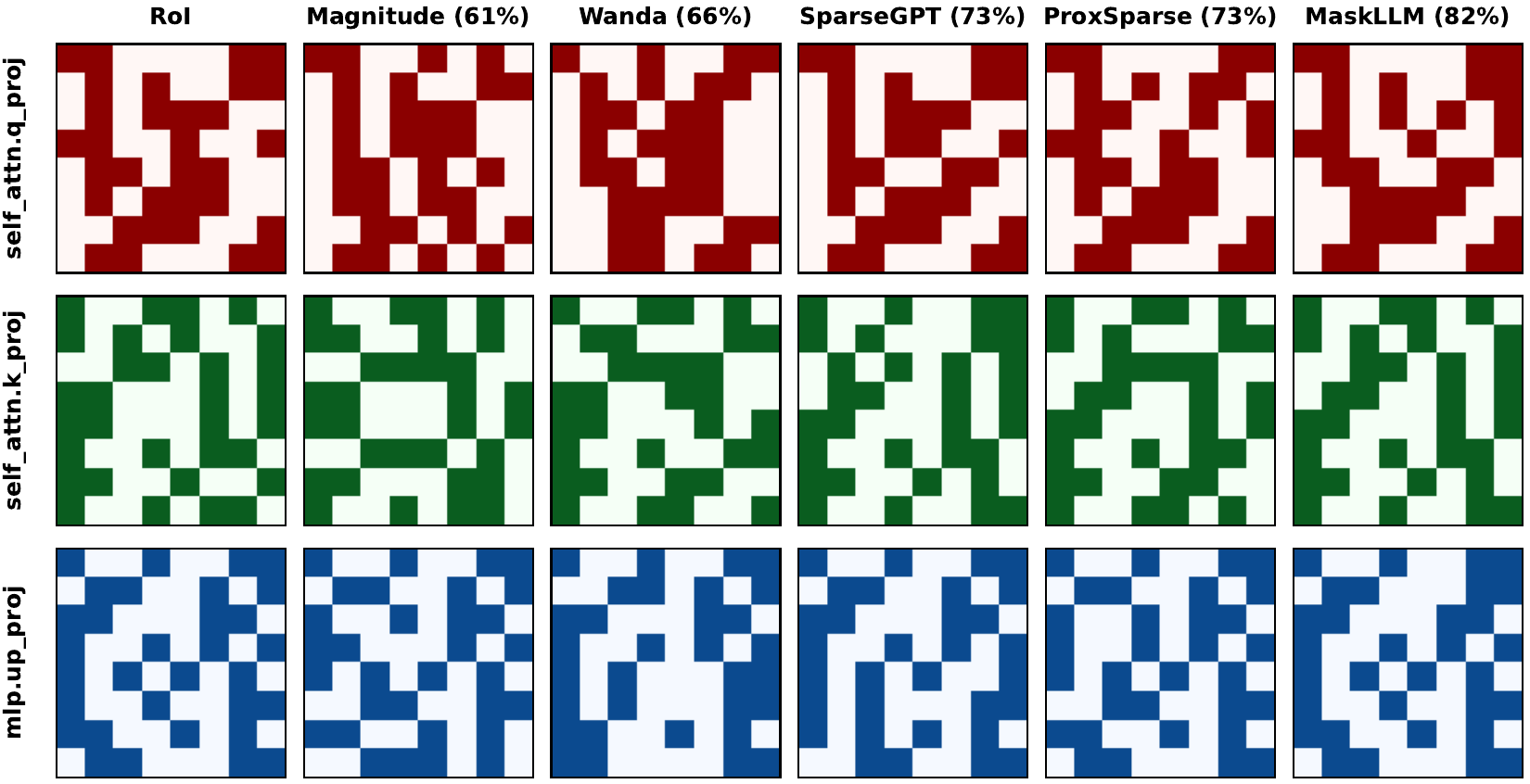}
    \label{fig:robustness}
    \caption{Mask difference comparison between RoI and previous works on the Qwen2.5-7B model. Besides the name of each baseline, place an overlapping percentage indicating the similarity of the produced masks between that baseline and RoI. Bold pixels indicate masks corresponding to retained weights.}
    \label{fig:mask_differences}
\end{figure*}

For instance, on Qwen2.5-1.5B, ST achieves 30.58 perplexity and 53.10\% average accuracy, while the soft approach (RoI) reaches 21.27 perplexity and 55.04\% accuracy. Similarly, on Qwen2.5-3B, the gap widens (23.57 vs. 19.37 perplexity). These observations suggest that the bias introduced by the ST estimator hampers generalization, outweighing the potential benefits of earlier discretization. Overall, the soft Gumbel-Top-$K$ relaxation used in RoI provides a more effective balance between trainability and performance.

\subsection{Throughput}
Table~\ref{tab:throughput} reports the decoding throughput of dense base models and their corresponding 2:4 sparse variants under different input/output token configurations. All models are deployed using vLLM with identical serving settings on a single NVIDIA B300 GPU to ensure a fair comparison.

\begin{table}[!h]
\begin{adjustbox}{max width=\linewidth}
\begin{tabular}{clcc}
\toprule
\textbf{I/O tokens}                 & \textbf{Models}         & \textbf{Throughput} & \textbf{Acceleration} \\
\midrule
\multirow{4}{*}{1024/4096} & Qwen2.5-3B     & 190.60     & 1.0x         \\
                           & Qwen2.5-3B 2:4 & 248.50     & 1.3x         \\ \cmidrule(lr){2-4}
                           & Qwen2.5-7B     & 165.00     & 1.0x         \\
                           & Qwen2.5-7B 2:4 & 210.10     & 1.27x        \\
\midrule
\multirow{4}{*}{2048/4096} & Qwen2.5-3B     & 185.20     & 1.0x         \\
                           & Qwen2.5-3B 2:4 & 241.40     & 1.3x         \\ \cmidrule(lr){2-4}
                           & Qwen2.5-7B     & 162.00     & 1.0x         \\
                           & Qwen2.5-7B 2:4 & 204.40     & 1.26x     
                              \\
\bottomrule
\end{tabular}
\end{adjustbox}
\caption{Throughput comparison between dense base models and their 2:4 sparse variants. All models are deployed using vLLM on a single NVIDIA B300 GPU. Throughput is reported under different input/output token configurations. The 2:4 sparse models consistently achieve higher throughput than their dense counterparts, demonstrating effective exploitation of semi-structured sparsity at inference time.}
\label{tab:throughput}
\end{table}

Across both Qwen2.5-3B and Qwen2.5-7B, the 2:4 sparse variants consistently achieve higher throughput than their dense counterparts. For the 1024/4096 I/O setting, the 2:4 models yield up to 1.3× speedup for Qwen2.5-3B and 1.27× speedup for Qwen2.5-7B. Similar improvements are observed under the 2048/4096 configuration, indicating that the throughput gains are robust across different sequence lengths.
These results demonstrate that semi-structured 2:4 sparsity can be effectively exploited at inference time, translating sparsity into tangible system-level acceleration without modifying the serving stack. Notably, the relative speedups remain stable across model scales, suggesting that the benefits of 2:4 sparsity generalize well from smaller to larger models.

\subsection{Mask Difference Analysis}
To investigate how different pruning strategies select weights, we measure the overlap between masks produced by various methods on the same model. 
Figure~\ref{fig:mask_differences} presents a qualitative comparison of the sparsity masks produced by RoI and several representative pruning baselines, including Magnitude, Wanda, SparseGPT, ProxSparse, and MaskLLM. For each method, we visualize the top-left 8$\times$8 sub-matrix of the learned masks from three representative layers: \texttt{self\_attn.q\_proj}, \texttt{self\_attn.k\_proj}, and \texttt{mlp.up\_proj}. Light pixels denote zero entries corresponding to pruned weights. In addition, we report the overlapping percentage between the masks generated by each baseline and RoI, indicating the degree of structural similarity.
Across all layers, RoI exhibits a clear alignment with MaskLLM, as reflected by the higher overlap ratios. This suggests that RoI is able to recover pruning patterns that are consistent with more expensive or data-intensive approaches, despite relying on a more efficient differentiable subset selection mechanism. In contrast, Magnitude pruning shows substantially lower overlap, indicating that purely weight-based criteria lead to structurally different sparsity patterns that deviate from those discovered by data-aware methods.
Notably, ProxSparse achieves an overlap comparable to SparseGPT in certain layers, but still demonstrates noticeable structural discrepancies, particularly in attention projections. This observation is consistent with ProxSparse’s deterministic and early-converging optimization behavior, which can restrict exploration of alternative sparsity configurations. Overall, this analysis highlights that RoI not only improves quantitative performance but also produces structured masks that closely resemble those of strong, data-aware pruning baselines, reinforcing its effectiveness in learning meaningful semi-structured sparsity patterns.

\begin{table*}[ht]
\centering
\begin{adjustbox}{max width=\textwidth}
\begin{tabular}{lcccccccc}
\toprule
\textbf{Method} & \textbf{W/U} & \textbf{ARC-E} & \textbf{BoolQ} & \textbf{HellaS.} & \textbf{PIQA} & \textbf{RACE} & \textbf{SciQ} & \textbf{Average} $\uparrow$ \\
\midrule
\rowcolor{headergray}
\textbf{Base Model: Qwen2.5-0.5B} & -    & 64.48          & 61.80          & 40.54          & 70.51          & 34.74          & 92.90          & 60.83          \\
Magnitude                         & \xmark  & 24.96          & \underline{38.72}    & 25.76          & 51.90          & 20.96          & 19.30          & 30.27          \\
Wanda                             & \xmark  & 26.22          & 37.83          & 26.43          & 54.84          & 23.54          & 29.40          & 33.04          \\
SparseGPT                         & \cmark & 28.37          & 38.56          & 26.65          & 53.05          & 24.11          & 41.70          & 35.41          \\
MaskLLM                           & \xmark  & \underline{42.75}    & \textbf{61.65} & \underline{26.92}    & \textbf{59.42} & \underline{25.86}    & \textbf{73.80} & \underline{48.40}    \\
\textbf{RoI (Ours)}              & \xmark  & \textbf{43.98} & \textbf{61.65} & \textbf{27.11} & \underline{59.36}    & \textbf{25.93} & \underline{73.10}    & \textbf{48.52} \\
\midrule
\rowcolor{headergray}
\textbf{Base Model: Qwen2.5-1.5B} & -    & 75.21          & 72.48          & 50.18          & 75.52          & 36.65          & 94.20          & 67.37          \\
Magnitude                         & \xmark  & 26.09          & 41.41          & 25.77          & 51.69          & 21.24          & 20.30          & 31.08          \\
Wanda                             & \xmark  & 25.76          & 47.98          & 25.65          & 51.63          & 23.44          & 22.00          & 32.74          \\
SparseGPT                         & \cmark & 28.54          & 54.31          & 26.52          & 53.16          & 24.69          & 45.20          & 38.74          \\
MaskLLM                           & \xmark  & \underline{50.12}    & \underline{58.87}    & \textbf{30.20} & \underline{63.32}    & \underline{27.80}    & \textbf{77.40} & \underline{51.29}    \\
\textbf{RoI (Ours)}              & \xmark  & \textbf{50.25} & \textbf{58.99} & \underline{29.97}    & \textbf{63.49} & \textbf{28.08} & \underline{77.20}    & \textbf{51.33} \\
\midrule
\rowcolor{headergray}
\textbf{Base Model: Qwen2.5-3B}   & -    & 77.57          & 77.22          & 55.04          & 78.29          & 38.47          & 95.90          & 70.42          \\
Magnitude                         & \xmark  & 24.24          & 37.83          & 25.59          & 52.23          & 20.77          & 19.70          & 30.06          \\
Wanda                             & \xmark  & 26.89          & 38.41          & 26.42          & 53.16          & 23.54          & 25.50          & 32.32          \\
SparseGPT                         & \cmark & 35.14          & 58.99          & 27.46          & 55.11          & 24.02          & 68.90          & 44.94          \\
MaskLLM                           & \xmark  & \underline{50.57}    & \underline{61.91}    & \textbf{30.70} & \underline{64.08}    & \underline{28.22}    & \textbf{82.00} & \underline{52.91}    \\
\textbf{RoI (Ours)}              & \xmark  & \textbf{51.01} & \textbf{62.32} & \underline{30.06}    & \textbf{64.69} & \textbf{28.80} & \underline{81.60}    & \textbf{53.08} \\
\midrule
\rowcolor{headergray}
\textbf{Base Model: Qwen2.5-7B}   & -    & 80.43          & 84.65          & 59.98          & 78.78          & 41.91          & 96.60          & 73.73          \\
Magnitude                         & \xmark  & 26.26          & 46.64          & 25.42          & 52.45          & 22.01          & 22.20          & 32.50          \\
Wanda                             & \xmark  & 27.48          & 37.83          & 26.35          & 53.59          & 22.68          & 32.60          & 33.42          \\
SparseGPT                         & \cmark & 32.91          & 59.33          & 27.87          & 55.28          & 27.37          & 75.20          & 46.33          \\
MaskLLM                           & \xmark  & \textbf{51.20} & \underline{60.98}    & \underline{29.32}    & \underline{62.90}    & \textbf{28.98} & \textbf{83.60} & \underline{52.83}    \\
\textbf{RoI (Ours)}              & \xmark  & \underline{50.67}    & \textbf{62.35} & \textbf{29.70} & \textbf{63.44} & \underline{28.85}    & \underline{83.40}    & \textbf{53.07} \\
\bottomrule
\end{tabular}
\end{adjustbox}
\caption{Comparative evaluation of zero-shot accuracy across multiple benchmark datasets for various pruning methods on the Qwen2.5 model family of different sizes with 2:8 sparsity pattern. \textbf{Bold} and \underline{underlined} values indicate the highest and second-highest performance, respectively. ProxSparse is excluded due to its inability to learn 2:8 sparsity patterns.}
\label{tab:acc_2:8}
\end{table*}

\subsection{Accuracy Results for 2:8 Sparsity}
\label{apx:acc_2:8}

Table~\ref{tab:acc_2:8_ARC-C} re-illustrates the severity of the 2:8 sparsity constraint. On ARC-Challenge, pruned models of all sizes and methods fail to exceed the 25\% random-guess baseline. Across the remaining benchmarks (Table~\ref{tab:acc_2:8}), heuristic pruning methods generally suffer substantial performance degradation, although SparseGPT retains non-trivial accuracy on several tasks. In contrast, approaches that explicitly learn the sparsity mask, namely RoI and MaskLLM, consistently achieve much higher average accuracy across model sizes.

\begin{table}[ht]
\centering
\begin{adjustbox}{max width=\linewidth}
\begin{tabular}{lcccc}
\toprule
& \multicolumn{4}{c}{\textbf{Model Sizes}} \\ \cmidrule(lr){2-5}
\textbf{Method}      & \textbf{0.5B} & \textbf{1.5B} & \textbf{3B} & \textbf{7B} \\
\midrule
\rowcolor{headergray}
\textbf{w/o Pruning} & 29.01         & 41.55         & 44.62       & 48.21       \\
Magnitude            & 20.99         & 21.93         & 23.21       & 21.50       \\
Wanda                & 19.80         & 19.62         & 19.37       & 19.88       \\
SparseGPT            & 20.48         & 19.80         & 17.92       & 18.69       \\
MaskLLM              & 18.88         & 21.51         & 20.81       & 21.41       \\
\textbf{RoI (Ours)} & 18.26         & 21.33         & 20.48       & 21.33       \\
\bottomrule
\end{tabular}
\end{adjustbox}
\caption{ARC-Challenge zero-shot accuracy scores across backbone models with 2:8 sparsity pattern setting. ProxSparse is excluded due to its inability to learn 2:8 sparsity patterns.}
\label{tab:acc_2:8_ARC-C}
\end{table}

Among saliency-based strategies, SparseGPT is the only method that preserves non-negligible performance, due to permitting parameter updates, and only on SciQ. These accuracy trends indicate that explicit mask learning is the most effective strategy. This observation is consistent with the 2:4 accuracy results in Table~\ref{tab:acc_2:4}, as well as the 2:4 and 2:8 perplexity results reported in Tables~\ref{tab:perplexity_PPL_2:4} and~\ref{tab:perplexity_PPL_2:8}. Collectively, these results further emphasize the advantages of RoI over MaskLLM, showing that RoI achieves competitive or superior performance at a fraction of the computational cost (Table~\ref{tab:num-parameters}), and remains robust across both tasks and sparsity regimes.

\subsection{Additional Results}

To evaluate whether the effectiveness of RoI generalizes beyond the Qwen2.5 model family, we further conduct experiments on Gemma2-9B~\cite{gemma2improving} under the 2:4 sparsity pattern. We compare RoI with both saliency-based methods and MaskLLM using WikiText-2 perplexity and zero-shot accuracy on seven downstream benchmarks. The results are reported in Table~\ref{tab:gemma2}.

\begin{table*}[ht]
\centering
\begin{adjustbox}{max width=\textwidth}
\begin{tabular}{l|c|cccccccc}
\toprule
\textbf{Method} & \textbf{PPL} $\downarrow$ & \textbf{ARC-C} & \textbf{ARC-E} & \textbf{BoolQ} & \textbf{HellaS.} & \textbf{PIQA} & \textbf{RACE} & \textbf{SciQ} & \textbf{Average} $\uparrow$ \\
\midrule
\rowcolor{headergray}
\textbf{Base Model: Gemma2-9B} & 10.54 & 61.52 & 87.16 & 84.16 & 61.20 & 81.18 & 42.49 & 97.20 & 73.56 \\
Magnitude & 537.39 & 34.39 & 68.31 & 69.79 & 42.46 & 70.78 & 33.21 & 88.10 & 58.15 \\
Wanda & 94.49 & 37.88 & 71.09 & 68.62 & 42.22 & 71.82 & 34.55 & 94.00 & 60.03 \\
SparseGPT & 68.54 & 36.35 & 69.49 & 72.54 & 44.25 & 71.82 & 36.17 & 93.70 & 60.62 \\
MaskLLM & \underline{17.25} & \underline{40.36} & \textbf{74.10} & \underline{73.08} & \underline{50.37} & \underline{75.41} & \underline{37.53} & \textbf{95.30} & \underline{63.73}  \\
RoI & \textbf{16.30} & \textbf{40.87} & \underline{73.82} & \textbf{73.55} & \textbf{50.98} & \textbf{75.79} & \textbf{38.28} & \textbf{95.30} & \textbf{64.08} \\
\bottomrule
\end{tabular}
\end{adjustbox}
\caption{Perplexity and zero-shot accuracy of different pruning methods on the Gemma-2-9B model under 2:4 sparsity. \textbf{Bold} and \underline{underlined} values indicate the best and second-best results among the pruning methods, respectively.}
\label{tab:gemma2}
\end{table*}

RoI achieves the lowest perplexity of 16.30, improving upon MaskLLM by 0.95 points and substantially outperforming the saliency-based baselines. It also obtains the highest average zero-shot accuracy of 64.08, compared with 63.73 for MaskLLM and 60.62 for the strongest saliency-based baseline, SparseGPT. Across individual tasks, RoI achieves the best or tied-best performance on six of the seven benchmarks, while remaining competitive on ARC-E. These results demonstrate that RoI transfers effectively to a different model family and consistently preserves language-modeling and downstream-task performance under semi-structured sparsity.

\end{document}